\documentclass[sigconf]{acmart}
\usepackage{courier}  
\usepackage{graphicx} 
\usepackage{natbib}  
\usepackage{caption} 
\usepackage{multirow}
\usepackage{soul}
\usepackage{color}
\usepackage{arydshln} 
\usepackage{amsthm}
\usepackage{algorithm}
\usepackage[noend]{algpseudocode} 
\usepackage{balance}

\usepackage[switch]{lineno} 

\newtheorem{myTheo}{Theorem}
\newtheorem{myDef}{Definition}
\newtheorem{myAssum}{Assumption}

\usepackage{newfloat} 
\usepackage{listings}
\DeclareCaptionStyle{ruled}{labelfont=normalfont,labelsep=colon,strut=off} 
\floatstyle{ruled}
\newfloat{listing}{tb}{lst}{}
\floatname{listing}{Listing}

\AtBeginDocument{%
  }

\copyrightyear{2025}
\acmYear{2025}
\setcopyright{cc}
\setcctype{by}
\acmConference[KDD '25]{Proceedings of the 31st ACM SIGKDD Conference on Knowledge Discovery and Data Mining V.2}{August 3--7, 2025}{Toronto, ON, Canada}
\acmBooktitle{Proceedings of the 31st ACM SIGKDD Conference on Knowledge Discovery and Data Mining V.2 (KDD '25), August 3--7, 2025, Toronto, ON, Canada}
\acmDOI{10.1145/3711896.3736890}
\acmISBN{979-8-4007-1454-2/2025/08}

\begin{document}

\title{Context Pooling: Query-specific Graph Pooling for Generic Inductive Link Prediction in Knowledge Graphs}


\author{Zhixiang Su}
\affiliation{%
  \institution{Nanyang Technological University}
  \country{Singapore}}
\email{zhixiang002@ntu.edu.sg}

\author{Di Wang}
\affiliation{%
  \institution{Nanyang Technological University}
  \country{Singapore}}
\email{wangdi@ntu.edu.sg}

\author{Chunyan Miao}
\authornote{Corresponding Author}
\affiliation{%
  \institution{Nanyang Technological University}
  \country{Singapore}}
\email{ascymiao@ntu.edu.sg}

\renewcommand{\shortauthors}{Zhixiang Su, Di Wang and Chunyan Miao}

\begin{abstract}
  Recent investigations on the effectiveness of Graph Neural Network (GNN)-based models for link prediction in Knowledge Graphs (KGs) show that vanilla aggregation does not significantly impact the model performance.  In this paper, we introduce a novel method, named Context Pooling, to enhance GNN-based models' efficacy for link predictions in KGs. To our best of knowledge, Context Pooling is the first methodology that applies graph pooling in KGs. Additionally, Context Pooling is first-of-its-kind to enable the generation of query-specific graphs for inductive settings, where testing entities are unseen during training. Specifically, we devise two metrics, namely neighborhood precision and neighborhood recall, to assess the neighbors' logical relevance regarding the given queries, thereby enabling the subsequent comprehensive identification of only the logically relevant neighbors for link prediction. Our method is generic and assessed by being applied to two state-of-the-art (SOTA) models on three public transductive and inductive datasets, achieving SOTA performance in 42 out of 48 settings.
\end{abstract}




\maketitle
\newcommand\kddavailabilityurl{https://doi.org/10.5281/zenodo.15502768}

\ifdefempty{\kddavailabilityurl}{}{
\begingroup\small\noindent\raggedright\textbf{KDD Availability Link:}\\
The source code of this paper has been made publicly available at \url{\kddavailabilityurl}.
\endgroup
}

\section{Introduction}\label{sec_introduction}
    Knowledge Graphs (KGs) are often used for organizing and representing a vast amount of structured information across diverse domains (e.g., commonsense, medical, and financial data~\cite{krotzsch2018attributed,lin2020kgnn,li2020real}).  A KG, represented as $G(E_G, R_G) = \{(h_i, r_i, t_i) | i = 1, 2, 3, ..., m\}$, is a heterogeneous graph of interconnected entities (heads and tails, $h_i, t_i \in E_G$) and relations ($r_i \in R_G$). In this context, link prediction is a critical process making inference of the missing entity in a given query triple. This involves the prediction of either the tail entity $t$ given $(h,r,?)$ or the head entity $h$ given $(?,r,t)$. To unlock the full potential of KGs, designated methods are required to perform link prediction, improve the completeness of KGs, and facilitate various downstream tasks (e.g., recommendation~\cite{wang2019kgat} and question answering~\cite{huang2019knowledge}). 
    
    \begin{figure}[!t]
        \centering
        \includegraphics[scale=0.10]{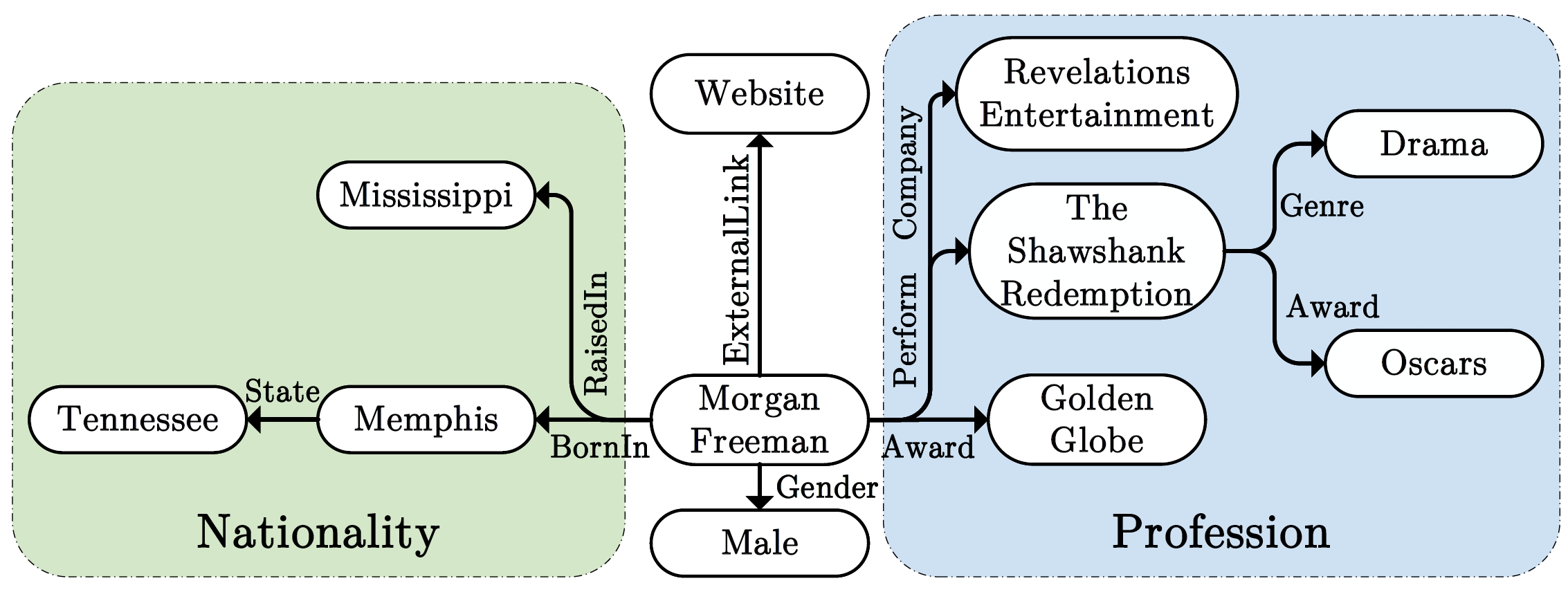}
        \caption{A KG comprising 2-hop neighbors of $\textit{MorganFreeman}$.}\label{fig1}
    \end{figure}
    
    In recent years, Graph Neural Network (GNN)-based models have emerged as a powerful tool for link prediction, especially for homogeneous graphs where entities and relations are of the same types~\cite{wu2020comprehensive}. A typical GNN architecture comprises two key components: \romannumeral1)~Message passing: Entities gather information by aggregating embeddings from their immediate neighboring entities; and \romannumeral2)~Entity updating: Entities adjust their features based on the information collected from neighbors, and possibly their own features.

    Traditionally, KGs are perceived as heterogeneous graphs that include diverse entity and relation types. In line with this perspective, GNN-based models assume that aggregating information from all neighboring nodes is beneficial and apply this principle to KGs~\cite{RGCN,RuleN}. Such practice is on the basis that relation-specific aggregations could be applied to deal with diverse relation types. However, recent studies~(\cite{zhang2022rethinking,li2022graph}) on the effectiveness of GNN components in link prediction posed challenges to this practice. These studies suggested that the aggregation of all neighbors may not significantly impact the performance on commonly used KGs. Surprisingly, replacing aggregation with a simple Multi-Layer Perceptron (MLP) neural network while keeping the other components static can produce comparable results~\cite{zhang2022rethinking}. These findings lead to the need of reexamining the efficacy of utilizing neighbor information in KGs.

    A major contributing factor to the outlined challenge is that when processing a specific query relation, only a limited number of neighbors are logically relevant, while the other neighbors are either irrelevant or illogical for link prediction. For instance, w.r.t Figure~\ref{fig1}, neighbors of \emph{MorganFreeman} can be categorized as follows:

    (\romannumeral1)~Logically relevant: These neighbors provide logically relevant supporting evidence regarding the query triple. In the context of the query (\emph{MorganFreeman}, \emph{Profession}, ?), neighbors connected to relations \emph{Company}, \emph{Perform}, and \emph{Award} are crucial for link prediction.
    
    (\romannumeral2)~Logically irrelevant: These neighbors offer useful information for the query entity \emph{MorganFreeman} but are logically irrelevant to the given query relation. For instance, relations \emph{BornIn} and \emph{RaisedIn} are helpful for the query (\emph{MorganFreeman}, \emph{Nationality}, ?) but unlikely to be useful for (\emph{MorganFreeman}, \emph{Profession}, ?).
    
    (\romannumeral3)~Illogical: These neighbors do not provide useful information for link prediction in the KG. For instance, \emph{ExternalLink} may lead to a website containing valuable semantic knowledge, but the link itself is irrelevant within the scope of link prediction in KGs.
    
     Although GNN-based models can adapt by allocating less attention to logically irrelevant and illogical neighbors, the inherent incompleteness and noise in KGs lead us to believe that selecting only the logically relevant neighbors could yield superior results. Neighbor pooling is a prominent research topic in this context, focusing on the selection of neighbors. However, to our best of knowledge, there does not exist any approach tailored for KGs. It is worth noting that a fraction of KGs are constructed with schema knowledge for selecting neighbors (e.g., YAGO~\cite{YAGO}). However, most KGs do not comprise such explicit information in practice. In this paper, we focus on designing a generic neighbor pooling method that is generally applicable to most GNN-based models and vast majority datasets. Specifically, we propose \textbf{Context Pooling} to address the outlined challenge by enabling GNN-based models to identify logically relevant neighbors. 
     The key properties of Context Pooling are as follows:

    \noindent \textbf{Quantifiable logical relevance:} We devise two metrics for Context Pooling, namely neighborhood precision and neighborhood recall. These two metrics help us to identify logically irrelevant and illogical neighbors by assessing their logical relevance to the query triple.

    \noindent \textbf{Generation of query-specific graphs:} To further avoid the identified neighbors being logically irrelevant, Context Pooling is designed to be query-specific. Starting from the given query head (or tail) entity, we develop algorithms (see Algorithms~\ref{algo_graph_construction} and~\ref{algo_optimized_context_family}) to generate query-specific graphs iteratively.

    \noindent \textbf{Inductive capability:} Recognizing the dynamic nature of real-world KGs, where testing entities are unseen during training, Context Pooling is designed with inductive capabilities, which can effectively deal with unseen entities and generate query-specific graphs for inductive settings.
    
    \noindent \textbf{Generic architecture:} Context Pooling is designed to be generic, so that it can be easily incorporated into existing GNN-based inductive relation prediction methods to further improve their accuracy (see Tables~\ref{table_transductive_relation_prediction} and~\ref{table_inductive_relation_prediction}).

    The key contributions of this work are as follows:
    
    (\romannumeral1)~We introduce Context Pooling, a novel and generic methodology to elevate the effectiveness of GNN-based link prediction models in KGs by identifying and utilizing logically relevant neighbors for specific queries. To our best of knowledge, Context Pooling is first-of-its-kind to apply graph pooling in KGs, which enables relatively lightweight generation of query-specific graphs for inductive settings.

    (\romannumeral2)~We devise two innovative metrics for Context Pooling, namely neighborhood precision and neighborhood recall, to identify neighbors of high logical relevance.

    (\romannumeral3)~We conduct transductive and inductive experiments to assess the effectiveness of Context Pooling. The experimental results show significant improvements in link prediction tasks across various KG datasets, achieving state-of-the-art (SOTA) performance in 42 out of 48 settings.

\section{Related Work}\label{sec_related_work}
    In this section, we first review GNN-based models for link prediction in KGs. Then, we introduce graph pooling, a key technique in homogeneous GNNs.
    
    \subsection{GNN-Based Link Prediction Models}

    GNN-based models applied in KGs utilize the graph's structure and the inherent knowledge within it. These models improve the representations of entities by aggregating and updating them using the embeddings of neighboring entities. The typical procedure comprises the following steps for each layer $l$ in the network:
    \begin{equation*}
    \begin{aligned}
        X^{l},W^{l}=\textit{Update}(\textit{Aggregation}(X^{l-1},W^{l-1}),X^{l-1},W^{l-1}), 
        \\l=1,2,...,L.
    \end{aligned}
    \end{equation*}
    where $X$ and $W$ denote the embedding matrices of entities and relations, respectively.
    
    For link prediction, the final layer embeddings $X^L$ and $W^L$ are utilized, applying a scoring function as follows:
    \begin{equation*}
        \textit{S}_{(h,r,t)}=\textit{Score}(X_h^L,W_r^L,X_t^L).
    \end{equation*}

    Typically, GNN-based models generate embeddings for entities and can well handle seen entities. Therefore, prior studies in this field primarily focus on transductive settings, where all entities are seen during training (e.g., pLogicNet ~\cite{pLogicNet} and DPMPN~\cite{DPMPN}).

    To extend relevant GNN methods to perform inductive link prediction, where testing entities are unseen during training, various strategies are used. For instance, a number of methods (e.g., GRAIL~\cite{GRAIL}) utilize local subgraphs to represent the relative positions of entities, creating initial embeddings for new entities based on their relative positions. Other methods (e.g., NBFNet~\cite{NBFNet} and RED-GNN~\cite{RED-GNN}) incorporate path information into GNNs, focusing on relation sequences that are also present in the testing graph, enabling processing without pre-existing entity embeddings.

    For prior studies afore-reviewed, neighbor information is considered to be effectively utilized by aggregation. However, recent studies~(\cite{zhang2022rethinking,li2022graph}) suggest that aggregation does not significantly impact the model performance. Therefore, in this paper, we introduce a novel approach named Context Pooling, to effectively utilize neighbor information in KGs.
    
    \subsection{Graph Pooling}\label{sec_rel_graph_pooling}

    Graph pooling plays an essential role in GNNs, condensing and streamlining the extensive information contained within large, intricate graphs. This process is comparable to pooling in standard neural networks, and is purposely adapted to handle the features in graph data structures. The graph pooling techniques in GNNs fall into two major categories~\cite{liu2022graphpooling}: flat pooling and hierarchical pooling.

    Flat pooling methods (e.g., sum-pooling and average-pooling) operate on the entire graph in a single step~\cite{duvenaud2015convolutional,xu2018powerful}. These methods do not change the graph's structure but rather aggregate information across the graph.
    In contrast, hierarchical pooling gradually reduces graph size, producing a tiered representation of the graph. This category generally encompasses two types: node cluster pooling and node drop pooling. The former (e.g., DiffPool~\cite{DiffPool} and MemPool~\cite{MemPool}) groups nodes into clusters and then pools the features within each cluster. Meanwhile, the latter (e.g., TopKPool~\cite{TopKPool} and IPool~\cite{IPool})  involves selectively dropping nodes based on certain criteria, effectively simplifying the graph.

    Our innovative Context Pooling method aligns with the node drop pooling approach, specifically designed to efficiently select neighboring nodes in KGs.

\section{Preliminary}
    Before introducing Context Pooling, we first present its preliminaries, including inductive link prediction, neighboring relation, and heterogeneous graph pooling.
    \begin{myDef}[Inductive Link Prediction]\label{def_inductive}
    Given a training KG $G_{\textit{train}} = (E_{G_{\textit{train}}}, R_G)$, a testing KG $G_{\textit{test}} = (E_{G_{\textit{test}}}, R_G)$, the query $(h_q,r_q,t_q)$ is inductive if:

    1) $E_{G_\textit{train}} \cap E_{G_\textit{test}}= \emptyset, h_q, t_q \in E_{G_\textit{test}}$,

    2) $R_{G_\textit{test}} \subseteq R_{G_\textit{train}}, r_q \in R_{G_\textit{train}}$.
    \end{myDef}
    
    Comparing to transductive link prediction, where the entities in $G_{\textit{test}}$ are a subset of those in $G_{\textit{train}}$ (i.e., $E_{G_{\textit{test}}} \subseteq  E_{G_{\textit{train}}}$), inductive prediction is designed to extend to entirely new, unseen entities~\cite{KRST}.

    \begin{myDef}[Neighboring Relation]
    Given a KG $G(E_G, R_G)$, the neighboring relation set of an entity $e$ is defined as follows:
    \begin{equation*}
        \textit{NR}_e=\{r'| \exists e_0 \in E_G, (e,r',e_0) \in G\}.
    \end{equation*}
    The neighboring relation set of a relation $r$ is defined as follows:
    \begin{equation*}
        \textit{NR}_r=\{r'| \exists e_1,e_2,e_3 \in E_G, (e_1,r',e_2)\in G \land (e_1,r,e_3) \in G\}.
    \end{equation*}
    \end{myDef}
    It is important to note that the concept of the neighboring relation set varies depending on the context. When discussing an entity $e$, it refers to the relations on the edges that connect to $e$. In the case of a relation $r$, it refers to other relations that share the same connected entity with $r$ in $G$.
    
    \subsection{Heterogeneous Graph Pooling}\label{sec_graph_pooling}
    Graph pooling methods were initially developed for homogeneous graphs. To extend these methods to KGs, we introduce the concept of heterogeneous graph pooling.
    \begin{myDef}[Heterogeneous Graph Pooling]
    \label{def_graph_pooling}
    Given a KG represented as $G(E_G, R_G)$, the graph pooling operation, denoted as $\textit{Pool()}$, transforms $G$ into a new, pooled graph $G'(E'_G, R'_G)$ in the following manner:
    \begin{equation*}
        G'(E'_G, R'_G)=\textit{Pool}(G(E_G, R_G)),
    \end{equation*}
    where $|E'_G| \leq |E_G|$ and $|R'_G| \leq |R_G|$.
    
    \end{myDef}

    The primary goal of graph pooling is to effectively reduce the number of nodes in a graph while preserving its essential structural information. Node drop pooling is a commonly applied method in this context, which is composed of the following three key components~\cite{liu2022graphpooling}:

    (\romannumeral1)~Score generator: The score generator evaluates each node's significance within the input graph using the feature matrix $X^l$ and the adjacency matrix $A^l$:
    \begin{equation*}
        S^l = \textit{Score}(X^l,A^l).
    \end{equation*}

    (\romannumeral2)~Node selector: The node selector then identifies and selects the nodes that have the highest significance scores, often referred to as the $\textit{Top}_k$ nodes as follows:
    \begin{equation*}
        \textit{idx}^{l+1} = \textit{Top}_k(S^l).
    \end{equation*}
    
    (\romannumeral3)~Graph Coarsening~\cite{cai2021graph}: Using the selected nodes, this process creates a coarsened version of the original graph as follows:
    \begin{equation*}
        X^{l+1}, A^{l+1} = \textit{Coarsen}(X^l, A^l, S^l, \textit{idx}^{l+1}).
    \end{equation*}

    However, the node drop pooling method faces challenges when being applied to inductive link prediction in KGs. Specifically, it struggles with unseen testing entities due to the difficulty in assigning feature matrices. Additionally, it requires training, which can be too costly and complex for large KGs, leading to difficulties in model convergence.

    To address these issues, we introduce Context Pooling, conceptually in line with node drop pooling. Context Pooling is designed to handle inductive settings and does not require training, offering a more efficient and generally applicable solution for graph pooling in KGs.
    
\section{Methodology}
    In this section, we first define two key metrics: neighborhood precision and neighborhood recall, which serve to quantify the logical relevance between neighbors and query relations. Subsequently, we introduce an iterative algorithm designed to construct a query-specific graph using logically relevant neighbors. Furthermore, we refine the iterative algorithm to be applicable to large KGs.  Lastly, we delineate how to enhance the commonly applied GNN-based models by integrating our Context Pooling technique for link predictions.

    \subsection{Logical Relevance Quantification}\label{sec_logical_relevance_quantification}

    Extending the discussion in Section~\ref{sec_introduction}, logically relevant neighbors provide supporting evidence for the query triple. However, relying on a singular neighbor often falls short in substantiating a specific relationship. For example, in Figure~\ref{fig1}, using just one neighbor (\emph{MorganFreeman}, \emph{Company}, \emph{RevelationsEntertainment}) does not conclusively reveal \emph{MorganFreeman}'s profession, considering a company has various positions. It is intuitively appropriate to involve a set of neighbors from all-round perspectives to achieve a more thorough understanding (see Definition~\ref{def_context_neighbor_family}).

    \begin{myDef}[Context Neighbor Family  (CNF)]\label{def_context_neighbor_family}
        Given a KG $G$ and a query relation $r$, the corresponding context neighbor family $\textit{CNF}()$ is defined as follows:
        \begin{equation}
            \textit{CNF}(r)= \{\textit{NR}'|\textit{Rel}(\textit{NR}',r)>R_0, \textit{NR}' \subseteq \textit{NR}_r \},
        \end{equation}
        where  $\textit{Rel}()$ is a scoring function to quantify logical relevance  and $R_0$ is a predefined threshold.
    \end{myDef}

    We conceptualize $\textit{CNF}$ as a family of sets, where each set has a relevance score ($\textit{Rel}()$) surpassing the threshold ($R_0$), qualifying them as both logically relevant and comprehensive. By using the term comprehensive, we highlight the importance of having an adequate number of relations within each individual set that logically substantiate a relationship. Furthermore, $\textit{CNF}()$ is designed to depend only on relations, which allows it to be effectively used in inductive scenarios where entities are unseen during training.
           
    To assess the logical relevance of these neighbor sets, we devise two metrics in Definitions~\ref{def_precision} and~\ref{def_recall}, respectively.

    \begin{myDef}[Neighborhood Precision]\label{def_precision}
    Given a KG $G(E_G, R_G)$, a query relation $r$ and a neighboring relation set $\textit{NR}'$, neighborhood precision is defined as follows:
    \begin{equation}
        \textit{Rel}^{\textit{pre}}(\textit{NR}',r)=\frac{|\{e_1^{\textit{pre}}| \textit{NR}' \cup \{ r \} \subseteq \textit{NR}_{e_1^{\textit{pre}}},e_1^{\textit{pre}} \in E_G\}| }{|\{e_2^{\textit{pre}}| \textit{NR}' \subseteq \textit{NR}_{e_2^{\textit{pre}}},e_2^{\textit{pre}} \in E_G\}|},
    \end{equation}
    where $|\cdot|$ denotes the set size.
        
    \end{myDef}

    Neighborhood precision quantifies how frequently the query relation $r$ appears in the neighborhood of entities having the neighboring relations $\textit{NR}'$. In practice, a set of comprehensive neighbors should exhibit high precision in $G_{\textit{train}}$, as the presence of these neighboring relations strongly suggests the likelihood of the query relation's existence.

    \begin{myDef}[Neighborhood Recall]\label{def_recall}
    Given a KG $G(E_G, R_G)$, a query relation $r$ and a neighboring relation set $\textit{NR}'$, neighborhood recall is defined as follows:
    \begin{equation}
        \textit{Rel}^{\textit{rec}}(\textit{NR}',r)=\frac{|\{e_1^{\textit{rec}}| \textit{NR}' \cup \{ r \} \subseteq \textit{NR}_{e_1^{\textit{rec}}},e_1^{\textit{rec}} \in E_G\}| }{|\{e_2^{\textit{rec}}| \{r\} \subseteq \textit{NR}_{e_2^{\textit{rec}}},e_2^{\textit{rec}} \in E_G\}|}.
    \end{equation}
    \end{myDef}

    Neighborhood recall assesses the frequency from a different perspective: It measures how often the neighbors $\textit{NR}'$ appear when the relation $r$ is present in the neighborhood. Practically, a set of logically relevant neighbors should demonstrate high recall in $G_{\textit{train}}$, indicating that when the query relation is observed in the neighborhood, these specific neighboring relations are also likely to be found.
 
    The procedure of generating a $\textit{CNF}$, as outlined in Algorithm~\ref{algo_context_familiy}, comprises the following key steps. Initially, we execute $\textit{CNFTrain}()$ on $G_{\textit{train}}$. $\textit{CNFTrain}()$ begins with the collation of the co-occurrence frequency between relations and sets of neighboring relations (see Lines~\ref{alg_line_count_begin}-\ref{alg_line_count_end}). Subsequently, $\textit{CNFTrain}()$ creates a power-set for each possible set of neighboring relations and assesses if their $\textit{Rel}()$ value exceeds a predefined threshold $R_0$ (see Lines~\ref{alg_line_CNF_begin}-\ref{alg_line_CNF_end}). Upon completing $\textit{CNFTrain}()$, we employ $\textit{CNFGenerate}()$ to identify neighbors for a specific query relation by comparing the current neighboring relations with those in $\textit{CNF}$ and selecting the one that shows the highest similarity among the identified logically relevant neighbors.

    \subsection{Query-specific Graph Construction}\label{sec_query_specific_graph_construction}

    In GNN-based models, which typically need multi-hop neighbor aggregation, constructing a query-specific graph using neighbors is crucial after defining $\textit{CNF}$. Our approach includes an iterative algorithm for this task, detailed in Algorithm~\ref{algo_graph_construction}.  Here, we focus on head entity queries because tail entity queries $(?,r,t)$ are equivalent to $(t, r_{\textit{inv}}, ?)$, where the subscript $\textit{inv}$ denotes the inversion of relation. Starting with a query $(h,r,?)$, we initially set $G_{\textit{context}}^{0}=\{(?,r_{\textit{inv}},h)\}$. We then iteratively produce $G_{\textit{context}}^{1}, G_{\textit{context}}^{2}, ..., G_{\textit{context}}^{L}$ for the required $L$-hop neighbors. 
    \begin{algorithm}[!t]
        \caption{CNF Generator}
        \label{algo_context_familiy}
        \begin{algorithmic}[1]
        \Procedure{$\textit{CNFTrain}$}{$G, \textit{Rel}(), R_0$}
            \State $\textit{r2N} = \textit{Dictionary}(\textit{Dictionary}())$
            \State $\textit{CNF} = \textit{Dictionary}(\textit{Set}())$
            \For{$e_i \in E_G$} \label{alg_line_count_begin}
                \State $\textit{NR}_{e_i} = \{r_j \mid (e_i, r_j, e_k) \in G, e_k \in E_G\}$
                \For{$r_i \in \textit{NR}_{e_i}$}
                    \State $\textit{r2N}[r{_i}][\textit{NR}_{e_i}] += 1$
                \EndFor
            \EndFor \label{alg_line_count_end}
    
            \For {$r_i \in R_G$} \label{alg_line_CNF_begin}
                \For{$\textit{NR}_i \in \textit{r2N}[r_i]$}
                    \State ${\textit{NF}= \textit{PowerSet}(\textit{NR}_i} - \{r_i\})$\label{alg_line_powerset}
                    \For{$\textit{NR}' \in \textit{NF}$}
                        \If{$\textit{Rel}(\textit{NR}',r_i)>R_0$}
                            \State $\textit{CNF}[r_i] = \textit{CNF}[r_i]+ \{\textit{NR}'\}$
                        \EndIf
                    \EndFor
                \EndFor
            \EndFor \label{alg_line_CNF_end}
            \State \textbf{return} $\textit{CNF}$
        \EndProcedure
        \Procedure{$\textit{CNFGenerate}$}{$G, \textit{CNF}, h, r$}
            \State $\textit{NR}_h = \{r_j \mid (h, r_j, e_k) \in G, e_k \in E_G\}$
            \State $\textit{MaxSim}=0$
            \State $\textit{NR}_{\textit{max}}=\textit{Set}()$
            \For{$\textit{NR}' \in \textit{CNF}[r]$}\label{alg_line_sim_start}
                \State $\textit{Sim}= \frac{|\textit{NR}' \cap \textit{NR}_h|}{|\textit{NR}' \cup \textit{NR}_h|}$\label{alg_line_sim_end}
                \If{$\textit{Sim}>\textit{MaxSim}$}
                    \State $\textit{MaxSim}=\textit{Sim}$
                    \State $\textit{NR}_{\textit{max}}=\textit{NR}' \cap \textit{NR}_h$
                \EndIf
            \EndFor
            \State \textbf{return} $\textit{NR}_{\textit{max}}$
        \EndProcedure
        \end{algorithmic}
    \end{algorithm}
        \begin{algorithm}[!t]
        \caption{Context Pooling}
        \label{algo_graph_construction}
        \begin{algorithmic}[1]
        \Procedure{$\textit{ContextPooling}$}{$G, G_{\textit{context}}^{l-1},\textit{CNF}$}
        \State $G_{\textit{context}}^{l}= \emptyset$
        \For{$(e^{l-2},r^{l-1},e^{l-1}) \in G_{\textit{context}}^{l-1}$}
            \State $r^{l-1}_{\textit{inv}}=\textit{InverseRelation}(r^{l-1})$
            \State $\textit{NR}_{\textit{max}}^{l}=\textit{CNFGenerate}(G,\textit{CNF},e^{l-1},r^{l-1}_{\textit{inv}})$
            \State $G' = \{(e^{l-1},r^l,e^l)| (e^{l-1},r^l,e^l) \in G \cap r^l \in \textit{NR}_{\textit{max}}^{l}\}$
            \State $G_{\textit{context}}^{l}=G_{\textit{context}}^{l}+G'$
        \EndFor
        \State \textbf{return} $G_{\textit{context}}^{l}$
        \EndProcedure
        \end{algorithmic}
    \end{algorithm}
    The key step of Algorithm~\ref{algo_graph_construction} is the generation of context neighbors for $e^{l-1}$ via $r^{l-1}_{\textit{inv}}$, $\forall (e^{l-2},r^{l-1},e^{l-1}) \in G_{\textit{context}}^{l-1}$. Essentially, $\textit{CNFGenerate}()$ can be regarded as a general procedure for selecting logically relevant and comprehensive neighbors for a target relation. In GNN-based aggregation for $(l-1)$-hop neighbors, the information of $e^{l-1}$ is aggregated through $r^{l-1}$.  Therefore, $r^{l-1}_{\textit{inv}}$ is exactly the relation required for generating $l$-hop neighbors using $\textit{CNFGenerate}()$.
    

    \subsection{Cost-effective Context Pooling}\label{sec_context_pooling}
    In this subsection, we optimize Context Pooling to enhance its adaptability to large KGs. Specifically, we reduce the complexity of generation and querying to $O(|R_G|^2)$ and $O(|\textit{CNF}^{\prime}|)$, respectively. This optimization greatly reduces Context Pooling's complexity, making the bottleneck of whether the extended GNN-based model can deal with large KGs to the model itself.
    
    Not all steps in Algorithms~\ref{algo_context_familiy} and~\ref{algo_graph_construction} are cost-effective when dealing with large KGs. The major bottlenecks are identified in Algorithm~\ref{algo_context_familiy} that 1)~the exponential growth ($O(2^{|R_G|})$) of the power-set size w.r.t the relation set size (see Line~\ref{alg_line_powerset}); and 2)~the computationally expensive task of computing similarity ($O(|\textit{CNF}||R_G|)$) during generation  (see Lines~\ref{alg_line_sim_start}-\ref{alg_line_sim_end}).
        
        \begin{figure*}[!t]
        \centering
        \includegraphics[scale=0.21]{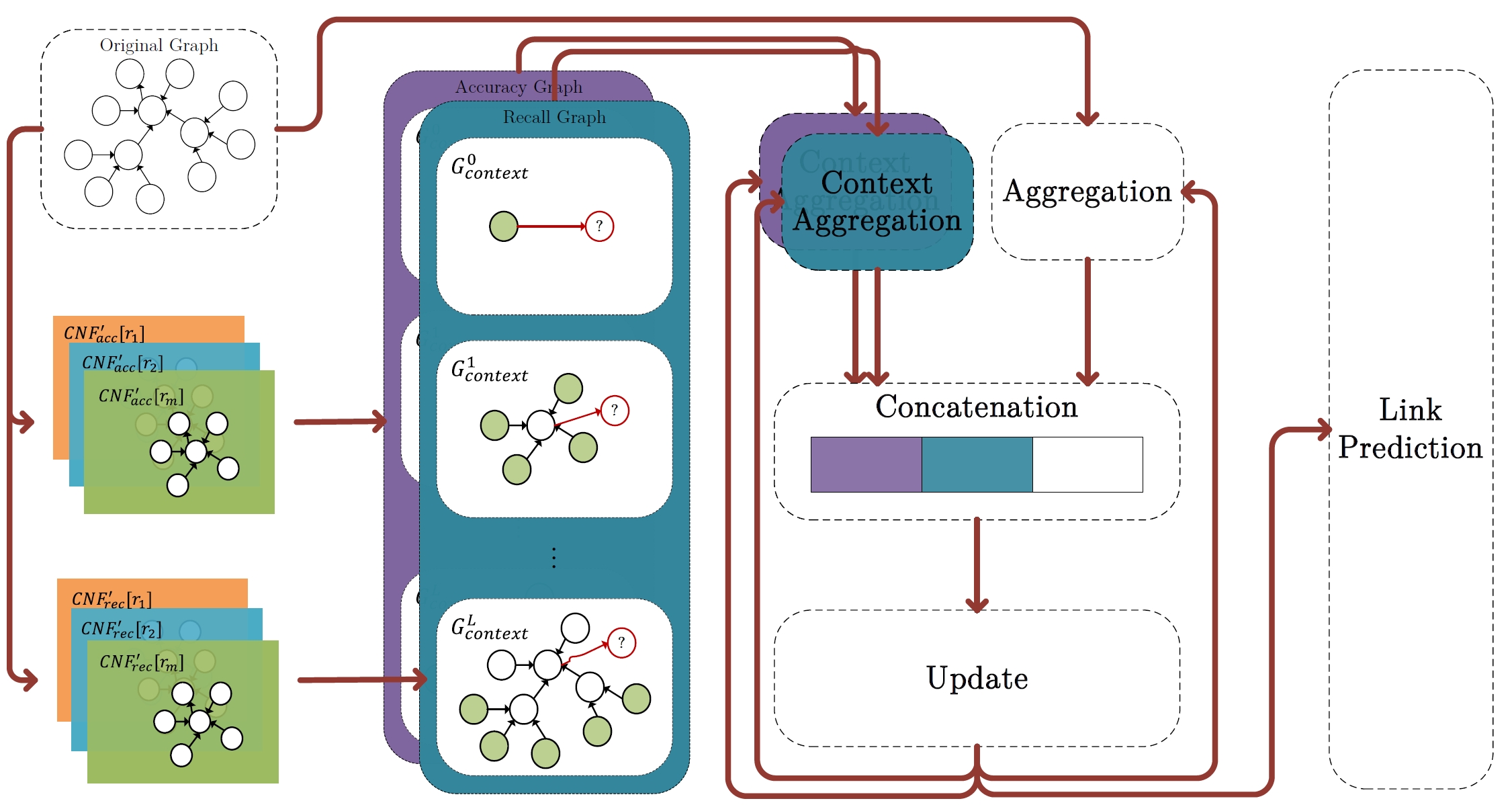}
        \caption{An illustration of Context Pooling GNN architecture.}\label{fig2}
    \end{figure*}

    To make Context Pooling efficient, especially for large KGs, we introduce the conditional independence assumption (Assumption~\ref{assum_independent}).
    This assumption posits that for any given query relation $r$, the occurrence probability of any two non-redundant relations $r_1, r_2 \in \textit{NR}_r$ is both independent and conditionally independent when considering $r$. Utilizing this assumption, we bypass the need to compute relevance scores for the entire power-set. Instead, we can independently compute the score for each relation in the neighborhood. Such alternative implementation streamlines the computational process and significantly reduces its cost.
        
    \begin{myAssum}\label{assum_independent}
    Given a query relation $r$, $\forall r_1, r_2 \in \textit{NR}_r, $ we assume the occurrence probability of $r_1$ and $r_2$ satisfies the following equations:
    \begin{equation}
    \begin{aligned}
        P(r_1,r_2)&=P(r_1)P(r_2), \\
        P(r_1,r_2|r)&=P(r_1|r)P(r_2|r).
    \end{aligned}
    \end{equation}
    \end{myAssum}

    Using Assumption~\ref{assum_independent}, we can adopt a Markov chain–inspired~\cite{norris1998markov} approach to show that when adding a new relation to an existing $\textit{CNF}$, its influence on the overall relevance score can be computed independently (see Theorem~\ref{theo_propto}).
    
    \begin{myTheo}\label{theo_propto}
        Given Assumption~\ref{assum_independent}, for a query relation $r$ and $\forall r_1, r_2 \in \textit{NR}_r$, the relationship between the combined and individual $\textit{Rel}^{\textit{pre}}$ and $\textit{Rel}^{\textit{rec}}$ can be expressed as follows:
        \begin{equation}\label{eq_propto_theo_1}
        \begin{aligned}
            \textit{Rel}^{\textit{pre}}(\{r_1, r_2\},r) \propto \textit{Rel}^{\textit{pre}}(\{r_1\},r)*\textit{Rel}^{\textit{pre}}(\{r_2\},r), \\
            \textit{Rel}^{\textit{rec}}(\{r_1, r_2\},r) \propto \textit{Rel}^{\textit{rec}}(\{r_1\},r)*\textit{Rel}^{\textit{rec}}(\{r_2\},r).
        \end{aligned}
        \end{equation}
        Furthermore, for any subset $\textit{NR}' \subseteq \textit{NR}_r$, the following can be deduced:
        \begin{equation}\label{eq_propto_theo_2}
        \begin{aligned}
            \textit{Rel}^{\textit{pre}}(\textit{NR}',r) \propto \prod_{r_i \in \textit{NR}'} \textit{Rel}^{\textit{pre}}(\{r_i\},r), \\
            \textit{Rel}^{\textit{rec}}(\textit{NR}',r) \propto \prod_{r_i \in \textit{NR}'} \textit{Rel}^{\textit{rec}}(\{r_i\},r).
        \end{aligned}
        \end{equation}
    \end{myTheo}
    

    The proof of Theorem~\ref{theo_propto} is provided in Appendix~\ref{sec_appendix_proof}. Theorem~\ref{theo_propto} allows us to much simplify the generation of $\textit{CNF}$. As shown in~(\ref{eq_propto_theo_2}), sets with high $\textit{Rel}$ values should consist of individual relations also having high $\textit{Rel}$ scores. This insight allows for an independent computation of each relation's score, leading to a more efficient implementation. Building on this theorem, we introduce an optimized version of Algorithm~\ref{algo_context_familiy}, detailed in Algorithm~\ref{algo_optimized_context_family}. The key changes include: 1)~employing $\textit{CNF}^{\prime}$ to archive individual neighboring relations with $\textit{Rel}$ scores surpassing a predetermined threshold ($O(|R_G|^2)$, see Lines~\ref{alg_line_opt_CNF_begin}-\ref{alg_line_opt_CNF_end}); and 2)~omitting the similarity computations for $\textit{CNF}^{\prime}$, while directly forming logically relevant neighbors through set intersection ($O(|\textit{CNF}^{\prime}|)$, see Line~\ref{alg_line_opt_sim}).

\begin{algorithm}[!t]
    \caption{Optimized CNF Generator}
    \label{algo_optimized_context_family}
    \begin{algorithmic}[1]
    \Procedure{$\textit{CNFTrain}^{\prime}$}{$G, \textit{Rel}(), R'_0$}
        \State $\textit{r2N} = \textit{Dictionary}(\textit{Dictionary}())$
        \State $\textit{CNF}^{\prime} = \textit{Dictionary}(\textit{Set}())$
        \For{$e_i \in E_G$} 
            \State $\textit{NR}_{e_i} = \{r_j \mid (e_i, r_j, e_k) \in G, e_k \in E_G\}$
            \For{$r_i \in \textit{NR}_{e_i}$}
                \State $\textit{r2N}[r{_i}][\textit{NR}_{e_i}] += 1$
            \EndFor
        \EndFor 

        \For {$r_i \in R_G$} 
            \For{$\textit{NR}_i \in \textit{r2N}[r_i]$}
                \State ${\textit{NF}'= (\textit{NR}_i} - \{r_i\})$
                \For{$r' \in \textit{NF}'$}\label{alg_line_opt_CNF_begin}
                    \If{$\textit{Rel}({r'},r_i)>R'_0$}
                        \State $\textit{CNF}^{\prime}[r_i] = \textit{CNF}^{\prime}[r_i]+ {r'}$
                    \EndIf
                \EndFor \label{alg_line_opt_CNF_end}
            \EndFor
        \EndFor 
        \State \textbf{return} $\textit{CNF}^{\prime}$
    \EndProcedure
    \Procedure{$\textit{CNFGenerate}^{\prime}$}{$G, \textit{CNF}^{\prime}, h, r$}
        \State $\textit{NR}_h = \{r_j \mid (h, r_j, e_k) \in G, e_k \in E_G\}$
        \State $\textit{NR}_{\textit{max}}= \textit{NR}_h \cap \textit{CNF}^{\prime}[r]$\label{alg_line_opt_sim}
        \State \textbf{return} $\textit{NR}_{\textit{max}}$
    \EndProcedure
    \end{algorithmic}
\end{algorithm}

    In summary, we define \textbf{Context Pooling} in line with the node drop pooling approach, structured as follows:

    (\romannumeral1)~Score generator: 
    \begin{equation*}
        S^l_{\textit{context}}(r^{l},r^{l-1})= 
            \textit{Rel}(\{r^{l}\}, r_{\textit{inv}}^{l-1}),   
    \end{equation*}
    where $(e^{l-2},r^{l-1},e^{l-1}) \in G_{\textit{context}}^{l-1}, r^{l} \in \textit{NR}_{e^{l-1}}$.

    (\romannumeral2)~Node selector:
    \begin{equation*}
        \textit{NR}_{\textit{max}}^{l}(r^{l-1})=\{r^l | S^l_{\textit{context}}(r^{l},r^{l-1})>R_0\}.
    \end{equation*}

    (\romannumeral3)~Graph Coarsening:
    \begin{equation*}
            \begin{aligned}
        G'(r^{l-1}) = \{(e^{l-1},r^l,e^l)|  (e^{l-1},r^l,e^l) \in G \cap r^l \in \textit{NR}_{\textit{max}}^{l}(r^{l-1})\},
    \end{aligned}
    \end{equation*}
    \begin{equation*}
            G_{\textit{context}}^{l} = \bigcup_{r^{l-1} \in R_{G_{\textit{context}}^{l-1}}} G'(r^{l-1}).
    \end{equation*}

    \subsection{Generic Context Pooling GNN Architecture}\label{sec_context_GNN}

    To integrate Context Pooling with GNN-based models for enhanced link prediction, we propose a generic framework (see Figure~\ref{fig2}).
    
    First, we generate the query-specific graph:
    \begin{equation}
                G_{\textit{context}}=\bigcup_{l=1,2,...,L} G_{\textit{context}}^{l}.
    \end{equation}
    
    We then perform aggregation on both the original and query-specific graphs, respectively:
    \begin{equation}
            \textit{Agg}^l=\textit{Aggregation}(G,X^{l-1},W^{l-1}),
    \end{equation}
    \begin{equation}  \textit{Agg}^l_{\textit{context}}=\textit{Aggregation}(G_\textit{context},X^{l-1},W^{l-1}).
    \end{equation}
    
    In the update phase, we concatenate the embeddings from both graphs and apply an update function:
    \begin{equation}
            X^l=\textit{Update}(\textit{Concat}(\textit{Agg}^l,\textit{Agg}^l_{\textit{context}}),X^{l-1},W^{l-1}).
    \end{equation}
    
    Lastly, we employ the final embedding obtained for making predictions:
    
    \begin{equation}
        \textit{S}_{(h,r,t)}=\textit{Score}(X_h^L,W_r^L,X_t^L).
    \end{equation}
    
\begin{table*}[!t]
\centering
\caption{Transductive link prediction results on  WN18RR, FB15k-237 and NELL-995}\label{table_transductive_relation_prediction}
\begin{tabular}{lcccccccccccc}
\toprule
           & \multicolumn{4}{c}{\textbf{WN18RR}}                                     & \multicolumn{4}{c}{\textbf{FB15k-237}}                                & \multicolumn{4}{c}{\textbf{NELL-995}}                                  \\
\textbf{}  & \textbf{V1}    & \textbf{V2}    & \textbf{V3}          & \textbf{V4}    & \textbf{V1}    & \textbf{V2}    & \textbf{V3}    & \textbf{V4}    & \textbf{V1}    & \textbf{V2}    & \textbf{V3}    & \textbf{V4}    \\
\midrule
\multicolumn{13}{c}{\textbf{MRR}}                                                                                                                                                                                            \\
\textbf{Neural LP}  & 0.549          & 0.580          & 0.518          & 0.550          & 0.304               & 0.364          & 0.347          & 0.328          & 0.374          & 0.302          & 0.259          & 0.268          \\
\textbf{DRUM}       & \underline{0.584}    & 0.585          & 0.519          & 0.574          & 0.317               & 0.377          & 0.357          & 0.333          & 0.413          & 0.320          & 0.274          & 0.292          \\
\textbf{CompGCN}    & 0.160          & 0.211          & 0.162          & 0.178          & 0.339               & 0.318          & 0.285          & 0.269          & 0.212          & 0.289          & 0.260          & 0.264          \\
\textbf{AnyBURL}    & 0.185          & 0.234          & 0.183          & 0.195          & 0.248               & 0.253          & 0.248          & 0.248          & 0.190          & 0.215          & 0.196          & 0.223          \\
\textbf{GRAIL}      & 0.455          & 0.486          & 0.212          & 0.448          & 0.252               & 0.253          & 0.218          & 0.226          & 0.346          & 0.363          & 0.323          & 0.296          \\
\textbf{NBFNet}     & 0.418          & 0.393          & 0.515          & 0.266          & 0.495               & 0.473          & 0.438          & 0.415          & 0.503          & 0.422          & 0.393          & 0.400          \\
\textbf{RED-GNN}    & 0.574          & \underline{0.597}    & \underline{0.561}    & \underline{0.581}    & 0.506               & \underline{0.505}    & \underline{0.474}    & \textbf{0.456} & \underline{0.516}    & \underline{0.503}    & \underline{0.466}    & \underline{0.475}    \\ \hdashline
\textbf{NBFNet+CP}  & 0.435          & 0.439          & 0.510          & 0.331          & \textbf{0.513}         & 0.480          & 0.447          & 0.421          & 0.515          & 0.449          & 0.407          & 0.420          \\
\textbf{RED-GNN+CP} & \textbf{0.594} & \textbf{0.612} & \textbf{0.573} & \textbf{0.593} & \underline{0.512}      & \textbf{0.508} & \textbf{0.478} & \underline{0.450}    & \textbf{0.524} & \textbf{0.516} & \textbf{0.468} & \textbf{0.484} \\
\midrule
\multicolumn{13}{c}{\textbf{Hit@1 (\%)}}                                                                                                                                                                                     \\
\textbf{Neural LP}  & 47.0           & 52.1           & 45.4           & 47.5           & 23.9                & 28.7           & 27.0           & 25.3           & 28.9           & 22.9           & 18.7           & 20.2           \\
\textbf{DRUM}       & 52.1           & 52.5           & 45.5           & 51.3           & 26.0                & 30.5           & 28.3           & 26.0           & 33.8           & 25.4           & 20.2           & 22.8           \\
\textbf{CompGCN}    & 13.1           & 17.8           & 13.7           & 15.3           & 29.2                & 24.7           & 21.6           & 19.5           & 16.1           & 22.0           & 17.8           & 20.1           \\
\textbf{AnyBURL}    & 16.1           & 19.5           & 15.8           & 16.6           & 22.3                & 19.2           & 18.7           & 18.6           & 16.2           & 15.1           & 13.2           & 16.8           \\
\textbf{GRAIL}      & 37.5           & 40.8           & 10.8           & 35.8           & 39.5                & 39.5           & 34.3           & 31.1           & 24.6           & 25.0           & 23.2           & 20.5           \\
\textbf{NBFNet}     & 31.9           & 29.3           & 43.4           & 18.8           & 40.4                & 37.5           & 33.9           & 31.5           & 41.6           & 32.2           & 29.1           & 30.1           \\
\textbf{RED-GNN}    & \underline{53.1}     & \underline{55.0}     & \underline{50.9}     & \underline{53.2}     & 43.7                & \underline{41.1}     & \underline{37.5}     & \textbf{35.5}  & \underline{42.3}     & \underline{40.5}     & \textbf{36.6}  & \underline{38.2}     \\\hdashline
\textbf{NBFNet+CP}  & 34.1           & 35.0           & 43.6           & 20.2           & \textbf{44.6} & 38.1           & 34.5           & 31.7           & \underline{42.3}     & 34.5           & 29.8           & 32.0           \\
\textbf{RED-GNN+CP} & \textbf{54.8}  & \textbf{56.5}  & \textbf{52.0}  & \textbf{54.4}  & \textbf{44.6} & \textbf{41.9}  & \textbf{37.6}  & \underline{35.1}     & \textbf{43.8}  & \textbf{41.8}  & \underline{36.4}     & \textbf{38.5}  \\
\bottomrule
\end{tabular}
\end{table*}

    \begin{table*}[htbp]
\centering
\caption{Inductive link prediction results on  WN18RR, FB15k-237 and NELL-995}\label{table_inductive_relation_prediction}
\begin{tabular}{lcccccccccccc}
\toprule
           & \multicolumn{4}{c}{\textbf{WN18RR}}                                     & \multicolumn{4}{c}{\textbf{FB15k-237}}                                & \multicolumn{4}{c}{\textbf{NELL-995}}                                  \\
\textbf{}  & \textbf{V1}    & \textbf{V2}    & \textbf{V3}          & \textbf{V4}    & \textbf{V1}    & \textbf{V2}    & \textbf{V3}    & \textbf{V4}    & \textbf{V1}    & \textbf{V2}    & \textbf{V3}    & \textbf{V4}    \\
\midrule
\multicolumn{13}{c}{\textbf{MRR}}                                                                                                                                                                                            \\
\textbf{Neural LP}  & 0.649          & 0.635          & 0.361                & 0.628          & 0.325          & 0.389          & 0.400          & 0.396          & 0.610          & 0.361          & 0.367          & 0.261          \\
\textbf{DRUM}       & 0.666          & 0.646          & 0.380                & 0.627          & 0.333          & 0.395          & 0.402          & 0.410          & \underline{0.628}    & 0.365          & 0.375          & 0.273          \\
\textbf{CompGCN}    & 0.681          & 0.653          & 0.407                & 0.603          & 0.350          & 0.451          & \underline{0.407}    & 0.383          & 0.487          & 0.403          & 0.347          & 0.333          \\
\textbf{AnyBURL}    & 0.391          & 0.548          & 0.340                & 0.550          & 0.355          & 0.437          & 0.338          & 0.318          & 0.477          & \underline{0.431}    & 0.391          & 0.332          \\
\textbf{GRAIL}      & 0.627          & 0.625          & 0.323                & 0.553          & 0.279          & 0.276          & 0.251          & 0.227          & 0.481          & 0.297          & 0.322          & 0.262          \\
\textbf{NBFNet}     & 0.684          & 0.652          & \textbf{0.425} & 0.604          & 0.307          & 0.369          & 0.331          & 0.305          & 0.544          & 0.410          & 0.425          & 0.287          \\
\textbf{RED-GNN}    & \underline{0.693}    & \underline{0.690}    & 0.422                & \textbf{0.651} & \textbf{0.318} & \underline{0.452}    & \underline{0.407}    & \underline{0.422}    & 0.618          & 0.381 & 0.393 & 0.342    \\ \hdashline
\textbf{NBFNet+CP}  & \underline{0.702}    & 0.658          & 0.393                & 0.605          & \textbf{0.387} & 0.423          & 0.400          & 0.377          & 0.560          & \textbf{0.442} & \textbf{0.457} & \underline{0.372}    \\
\textbf{RED-GNN+CP} & \textbf{0.708} & \textbf{0.694} & \textbf{0.425} & \underline{0.648}    & \underline{0.383}    & \textbf{0.463} & \textbf{0.447} & \textbf{0.427} & \textbf{0.634} & 0.430          & \underline{0.426}    & \textbf{0.382} \\
\midrule
\multicolumn{13}{c}{\textbf{Hit@1 (\%)}}                                                                                                                                                                                     \\
\textbf{Neural LP}  & 59.2           & 57.5           & 30.4                 & 58.3           & 24.3           & 28.6           & 30.9           & 28.9           & 50.0           & 24.9           & 26.7           & 13.7           \\
\textbf{DRUM}       & 61.3           & 59.5           & 33.0                 & 58.6           & 24.7           & 28.4           & 30.8           & 30.9           & 50.0           & 27.1           & 26.2           & 16.3           \\
\textbf{CompGCN}    & 62.0           & 59.0           & 36.0                 & 54.5           & 29.4           & \underline{36.2}     & 31.5           & 29.5           & 38.5           & 30.9           & 25.3           & 23.6           \\
\textbf{AnyBURL}    & 23.3           & 41.7           & 26.4                 & 45.8           & \underline{30.1}     & 33.7           & 25.5           & 23.3           & 38.6           & \textbf{33.0}  & 30.7           & 23.6           \\
\textbf{GRAIL}      & 55.4           & 54.2           & 27.8                 & 44.3           & 20.5           & 20.2           & 16.5           & 14.3           & 42.5           & 19.9           & 22.4           & 15.3           \\
\textbf{NBFNet}     & 59.2           & 57.5           & 30.4                 & 57.4           & 19.0           & 22.9           & 20.6           & 18.5           & 50.0           & 27.1           & 26.2           & 23.3           \\
\textbf{RED-GNN}    & \underline{64.6}     & \underline{63.3}     & \underline{36.4}           & \textbf{60.6}  & 24.7           & 35.6           & \underline{32.2}     & \underline{32.7}     & 50.0  & 28.8           & 30.0  & 23.8   \\ \hdashline
\textbf{NBFNet+CP}  & 62.2           & 57.8           & 32.9                 & 53.7           & 28.5           & 28.3           & 26.9           & 24.4           & \textbf{52.0}  & 31.2           & \textbf{35.0}  & \underline{25.7}     \\
\textbf{RED-GNN+CP} & \textbf{66.0}  & \textbf{63.7}  & \textbf{36.5}        & \underline{60.3}     & \textbf{31.6}  & \textbf{36.8}  & \textbf{35.8}  & \textbf{33.0}  & \underline{51.7}     & \underline{32.9}     & \underline{32.4}     & \textbf{27.8}  \\
\bottomrule
\end{tabular}
\end{table*}

\section{Experiments}

    We adopt four versions of three datasets for both transductive and inductive experimental settings, released by~\cite{GRAIL}. These subsets of data were derived from WN18RR~\cite{WN18RR}, FB15k-237~\cite{FB15k237}, and NELL-995~\cite{NELL995}, including large KGs such as WN18RR-V3, FB15k-237-V4, and NELL-995-V3 (see Appendix~\ref{sec_dataset_appendix}). Rather than the setup applied in ~\cite{GRAIL}, where 49 candidate entities were randomly selected as negative entities for testing, we apply a more challenging configuration introduced by~\cite{RED-GNN}. Specifically, in the testing phase, each query triple designates the ground-truth entity as positive, while all the other entities in the KG are treated as negative entities. 

    We apply Context Pooling onto two SOTA inductive link prediction approaches, namely NBFNet~\cite{NBFNet} and RED-GNN~\cite{RED-GNN}, to assess whether their performance NBFNet+CP and RED-GNN+CP is elevated\footnote{All datasets and code are available online: https://github.com/ZhixiangSu/Context-Pooling/}. We adopt the hyper-parameter values from the published NBFNet and RED-GNN models for a fair comparison. The precision and recall thresholds for Context Pooling are set within the range of $[10^{-5},10^{-1}]$. All experiments are conducted using PyTorch on two NVIDIA Tesla A100 GPUs with 40GB of RAM.

    In all the following transductive and inductive experiments, we apply Algorithm~\ref{algo_optimized_context_family} to demonstrate the efficacy of Context Pooling. Additionally, we conduct a case study to show the logically relevant neighbors on FB15k-237-V4 (see Section~\ref{sec_case_study_appendix}). 
    Furthermore, to compare the performance of the optimized algorithm (i.e., Algorithm~\ref{algo_optimized_context_family}) with its unoptimized counterpart (i.e., Algorithm~\ref{algo_context_familiy}), we conduct an ablation study in Section~\ref{sec_ablation_study_appendix}.

    We benchmark NBFNet+CP and RED-GNN+CP against SOTA inductive GNN-based and path-based models, including  Neural~LP~\cite{NeuralLP}, DRUM~\cite{DRUM}, CompGCN~\cite{CompGCN}, AnyBURL~\cite{AnyBURL} and GRAIL~\cite{GRAIL}. Traditional graph pooling approaches cannot be easily integrated with the GNN-based models to perform link prediction (see Section~\ref{sec_related_work}), thus, are not compared in this work.

\begin{table*}[ht]
    \centering
    \caption{Examples of Context Neighbors extracted from FB15k-237-V4}\label{table_case_study_appendix}
    \begin{tabular}{lcc}
    \toprule
    Query Relation                      & Context Neighbor - Precision    & Context Neighbor - Recall              \\ \midrule \midrule
    $\textit{award\_honor/award\_winner}$ & $\textit{award\_category/category\_of}$     & $\textit{award\_category/category\_of}$          \\
                                        & $\textit{award\_honor/ceremony}$           & $\textit{award\_honor/ceremony}$                \\
                                        & $\textit{award\_nomination/award}^{-1}$    & $\textit{award\_nomination/nominated\_for}$      \\
                                        & $\textit{award\_honor/award}^{-1}$         &                                                \\ \midrule
    $\textit{film/language}$            & $\textit{film/film\_art\_direction\_by}$     & $\textit{award\_nomination/nominated\_for}^{-1}$ \\
                                        & $\textit{award\_nomination/nominated\_for}$ & $\textit{actor/performance/film}^{-1}$         \\
                                        &                                           & $\textit{film/estimated\_budget/currency}$      \\
                                        &                                           & $\textit{film/genre}$                          \\
                                        &                                           & $\textit{film/film\_release\_region}$            \\
                                        &                                & ...                                   \\  \midrule
    $\textit{location/contains}$        & $\textit{bibs\_location/country}^{-1}$     & $\textit{location/contains}^{-1}$              \\
                                        & $\textit{country/second\_level\_divisions}$ & $\textit{/location/location/time\_zones}$       \\
                                        & $\textit{location/partially\_contain}$     & $\textit{place\_lived/location}^{-1}$           \\
                                        &                                           & $\textit{/film/film/country}^{-1}$               \\ 
                                        &                                & ...                                   \\ \bottomrule
    \end{tabular}
    \end{table*}

    \subsection{Experimental Results}
    
    \subsubsection{Transductive Link Prediction}

    In transductive link prediction, we divide the training triples into two parts. One-third of randomly selected triples are used as established facts for creating the training graph, while the remaining two-thirds training triples are utilized as queries during training. In the testing phase, all query entities and relations are seen during training.

    Table~\ref{table_transductive_relation_prediction} reports the transductive link prediction results. The highest scores representing the best performance are highlighted in bold, while the second-bests are underlined. Notably, RED-GNN+CP outperforms all the others in 10 out of 12 settings in terms of MRR and in 10 out of 12 settings in terms of Hit@1. In the remaining 3 of 4 settings, RED-GNN+CP is ranked as second-best. The most significant performance elevation is observed regarding WN18RR-V2 comparing to NBFNet, with an 11.7\% increase in MRR and a 19.4\% increase in Hit@1.

    
    \subsubsection{Inductive Link Prediction}

    In inductive link prediction, there is zero overlap of entities between the training and testing graphs. Furthermore, all entities in the testing queries are unseen during training.

    The results for inductive link prediction are reported in Table~\ref{table_inductive_relation_prediction}. RED-GNN+CP outperforms all the others in 8 out of 12 settings in terms of MRR and in 8 out of 12 settings in terms of Hit@1. In the remaining 5 of 8 settings, NBFNet+CP achieves the best performance. The most significant performance elevation is observed regarding NELL-995-V4, with an 11.7\% increase in MRR and a 16.8\% increase in Hit@1. 

    \subsubsection{Paired t-test}
    Comparing to both NBFNet and RED-GNN, incorporating Context Pooling leads to performance elevation in 88 of 96 transductive and inductive settings. For transductive settings and MRR in inductive settings, the paired t-tests show that p-values are less than $ 10^{-3}$.  In the case of Hit@1 for inductive settings, the p-value is derived as $ 2 \times 10^{-3}$. Such results underscore Context Pooling's capability of elevating performance in both transductive and inductive link prediction applications.

    \subsection{Case Study}\label{sec_case_study_appendix}
    To demonstrate the logical relevance between logically relevant neighbors and the query relation, we present the representative case study results on the transductive dataset FB15k-237-V4 comprising a large number of relations in Table~\ref{table_case_study_appendix}.
    
    As shown in Table~\ref{table_case_study_appendix}, for a query regarding award winners, both neighborhood precision and neighborhood recall (see Definitions~\ref{def_precision} and~\ref{def_recall}) are focused on award-related aspects, including award categories, ceremonies, and nomination specifics.  In the context of film-related queries, neighbors encompass elements such as art direction, awards, actors, budget, genre, and release regions. For queries concerning locations, neighbors are identified as countries, regional divisions, time zones, and residential areas.
    
    Such case study results suggest that in datasets involving a large number of relations (219 for FB15k-237-V4), Context Pooling can effectively retain a much smaller number of logically relevant neighbors, while filtering out those logically irrelevant or illogical ones, to elevate the link prediction performance (see Tables~\ref{table_transductive_relation_prediction} and~\ref{table_inductive_relation_prediction}).
\subsection{Ablation Study}\label{sec_ablation_study_appendix}
    To demonstrate the difference in performance and efficiency between the optimized and unoptimized $\textit{CNF}$ generation algorithms (see Algorithms~\ref{algo_context_familiy} and~\ref{algo_optimized_context_family}), we conduct an ablation study on two relatively smaller inductive datasets using RED-GNN+CP: FB15k-237-V1 and WN18RR-V1 (see Table~\ref{table_ablation_study}).
    We choose smaller datasets because the performance gap will definitely increase significantly with larger datasets.
    
    As discussed in Section~\ref{sec_context_pooling}, the computational time required for Algorithm~\ref{algo_context_familiy} grows exponentially with the quantity of relations ($O(2^{|R_G|})$). To make it more feasible in practical scenarios, we restrict the number of relations to $[4,6)$. This restriction reduces the complexity to $O(\Sigma_{i=4,5} C_{|R_G|}^i)$, where $C$ represents the combinatorial number. Nevertheless, this complexity is still significantly higher than that of Algorithm~\ref{algo_optimized_context_family} ($O(|R_G|^2)$). 

    On WN18RR-V1, the unoptimized Algorithm~\ref{algo_context_familiy} does not surpass the optimized Algorithm~\ref{algo_optimized_context_family} in terms of performance (-1.0\% for MRR and -1.2\% for Hit@1). This may be caused by noise and data bias within the dataset. On FB15k-237-V1, the unoptimized algorithm shows a higher level of performance (+3.4\% for MRR and +5.1\% for Hit@1). Notably, for both datasets, the optimized algorithm demonstrates a significantly lower time consumption, namely 14.7 times faster on WN18RR-V1 and 20.6 times faster on FB15k-237-V1. Therefore, we hold the view that the optimized algorithm can significantly reduce running time while maintaining a relatively high level of performance, enabling it to efficiently deal with relatively larger KGs.

    \begin{table}[!t]
    \small
    \centering
    \caption{Ablation study results}\label{table_ablation_study}
    \begin{tabular}{llcc} \toprule
           &             & WN18RR-V1 & FB15k-237-V1 \\ \midrule \midrule
MRR                                                                    & Optimized   & 0.708     & 0.383        \\
       & Unoptimized & 0.701     & 0.396        \\ \midrule
\multirow{2}{*}{\begin{tabular}[c]{@{}l@{}}Hit@1 \\ (\%)\end{tabular}} & Optimized   & 0.660     & 0.316        \\
       & Unoptimized & 0.652     & 0.332        \\ \midrule
\multirow{2}{*}{\begin{tabular}[c]{@{}l@{}}Time \\ (s)\end{tabular}}   & Optimized   & 412       & 1,197        \\
                                                                       & Unoptimized & 6,452     & 24,669        \\
\bottomrule
    \end{tabular}
\end{table}

\section{Conclusion}
    
    This paper presents Context Pooling for enhancing GNN-based link prediction in KGs. Context Pooling boosts performance by identifying logically relevant neighbors through neighborhood precision and neighborhood recall and forming query-specific graphs adaptable to unseen entities. Context Pooling is generic and is applied to enhance two SOTA methods to show its effectiveness. Experimental results demonstrate the superior performance of our proposed Context Pooling method in link prediction tasks across various KG datasets.
    
    Going forward, we plan to apply Context Pooling to domain-specific knowledge graphs, including healthcare, finance, and social networks.

\section*{Acknowledgment}
    This research is supported, in part, by the National Research Foundation, Prime Minister's Office, Singapore under its NRF Investigatorship Programme (NRFI Award No. NRF-NRFI05-2019-0002). Any opinions, findings and conclusions or recommendations expressed in this material are those of the author(s) and do not reflect the views of National Research Foundation, Singapore. This research is also supported, in part, by the RIE2025 Industry Alignment Fund – Industry Collaboration Projects (IAF-ICP) (Award I2301E0026), administered by A*STAR, as well as supported by Alibaba Group and NTU Singapore through Alibaba-NTU Global e-Sustainability CorpLab (ANGEL).

\bibliographystyle{ACM-Reference-Format}
\bibliography{ref}


\begin{thebibliography}{33}


\ifx \showCODEN    \undefined \def \showCODEN     #1{\unskip}     \fi
\ifx \showISBNx    \undefined \def \showISBNx     #1{\unskip}     \fi
\ifx \showISBNxiii \undefined \def \showISBNxiii  #1{\unskip}     \fi
\ifx \showISSN     \undefined \def \showISSN      #1{\unskip}     \fi
\ifx \showLCCN     \undefined \def \showLCCN      #1{\unskip}     \fi
\ifx \shownote     \undefined \def \shownote      #1{#1}          \fi
\ifx \showarticletitle \undefined \def \showarticletitle #1{#1}   \fi
\ifx \showURL      \undefined \def \showURL       {\relax}        \fi
\providecommand\bibfield[2]{#2}
\providecommand\bibinfo[2]{#2}
\providecommand\natexlab[1]{#1}
\providecommand\showeprint[2][]{arXiv:#2}

\bibitem[Baek et~al\mbox{.}(2021)]%
        {TopKPool}
\bibfield{author}{\bibinfo{person}{Jinheon Baek}, \bibinfo{person}{Minki Kang}, {and} \bibinfo{person}{Sung~Ju Hwang}.} \bibinfo{year}{2021}\natexlab{}.
\newblock \showarticletitle{Accurate learning of graph representations with graph multiset pooling}.
\newblock \bibinfo{journal}{\emph{arXiv preprint arXiv:2102.11533}} (\bibinfo{year}{2021}).
\newblock


\bibitem[Cai et~al\mbox{.}(2021)]%
        {cai2021graph}
\bibfield{author}{\bibinfo{person}{Chen Cai}, \bibinfo{person}{Dingkang Wang}, {and} \bibinfo{person}{Yusu Wang}.} \bibinfo{year}{2021}\natexlab{}.
\newblock \showarticletitle{Graph coarsening with neural networks}.
\newblock \bibinfo{journal}{\emph{arXiv preprint arXiv:2102.01350}} (\bibinfo{year}{2021}).
\newblock


\bibitem[Dettmers et~al\mbox{.}(2018)]%
        {WN18RR}
\bibfield{author}{\bibinfo{person}{Tim Dettmers}, \bibinfo{person}{Pasquale Minervini}, \bibinfo{person}{Pontus Stenetorp}, {and} \bibinfo{person}{Sebastian Riedel}.} \bibinfo{year}{2018}\natexlab{}.
\newblock \showarticletitle{Convolutional 2d knowledge graph embeddings}. In \bibinfo{booktitle}{\emph{Proceedings of the AAAI conference on artificial intelligence}}, Vol.~\bibinfo{volume}{32}.
\newblock


\bibitem[Duvenaud et~al\mbox{.}(2015)]%
        {duvenaud2015convolutional}
\bibfield{author}{\bibinfo{person}{David~K Duvenaud}, \bibinfo{person}{Dougal Maclaurin}, \bibinfo{person}{Jorge Iparraguirre}, \bibinfo{person}{Rafael Bombarell}, \bibinfo{person}{Timothy Hirzel}, \bibinfo{person}{Al{\'a}n Aspuru-Guzik}, {and} \bibinfo{person}{Ryan~P Adams}.} \bibinfo{year}{2015}\natexlab{}.
\newblock \showarticletitle{Convolutional networks on graphs for learning molecular fingerprints}.
\newblock \bibinfo{journal}{\emph{Advances in neural information processing systems}}  \bibinfo{volume}{28} (\bibinfo{year}{2015}).
\newblock


\bibitem[Gao et~al\mbox{.}(2021)]%
        {IPool}
\bibfield{author}{\bibinfo{person}{Xing Gao}, \bibinfo{person}{Wenrui Dai}, \bibinfo{person}{Chenglin Li}, \bibinfo{person}{Hongkai Xiong}, {and} \bibinfo{person}{Pascal Frossard}.} \bibinfo{year}{2021}\natexlab{}.
\newblock \showarticletitle{ipool—information-based pooling in hierarchical graph neural networks}.
\newblock \bibinfo{journal}{\emph{IEEE Transactions on Neural Networks and Learning Systems}} \bibinfo{volume}{33}, \bibinfo{number}{9} (\bibinfo{year}{2021}), \bibinfo{pages}{5032--5044}.
\newblock


\bibitem[Huang et~al\mbox{.}(2019)]%
        {huang2019knowledge}
\bibfield{author}{\bibinfo{person}{Xiao Huang}, \bibinfo{person}{Jingyuan Zhang}, \bibinfo{person}{Dingcheng Li}, {and} \bibinfo{person}{Ping Li}.} \bibinfo{year}{2019}\natexlab{}.
\newblock \showarticletitle{Knowledge graph embedding based question answering}. In \bibinfo{booktitle}{\emph{Proceedings of the twelfth ACM international conference on web search and data mining}}. \bibinfo{pages}{105--113}.
\newblock


\bibitem[Kr{\"o}tzsch et~al\mbox{.}(2018)]%
        {krotzsch2018attributed}
\bibfield{author}{\bibinfo{person}{Markus Kr{\"o}tzsch}, \bibinfo{person}{Maximilian Marx}, \bibinfo{person}{Ana Ozaki}, {and} \bibinfo{person}{Veronika Thost}.} \bibinfo{year}{2018}\natexlab{}.
\newblock \showarticletitle{Attributed description logics: Reasoning on knowledge graphs.}. In \bibinfo{booktitle}{\emph{International Joint Conferences on Artificial Intelligence}}. \bibinfo{pages}{5309--5313}.
\newblock


\bibitem[Li et~al\mbox{.}(2022)]%
        {li2022graph}
\bibfield{author}{\bibinfo{person}{Juanhui Li}, \bibinfo{person}{Harry Shomer}, \bibinfo{person}{Jiayuan Ding}, \bibinfo{person}{Yiqi Wang}, \bibinfo{person}{Yao Ma}, \bibinfo{person}{Neil Shah}, \bibinfo{person}{Jiliang Tang}, {and} \bibinfo{person}{Dawei Yin}.} \bibinfo{year}{2022}\natexlab{}.
\newblock \showarticletitle{Are graph neural networks really helpful for knowledge graph completion?}
\newblock \bibinfo{journal}{\emph{arXiv preprint arXiv:2205.10652}} (\bibinfo{year}{2022}).
\newblock


\bibitem[Li et~al\mbox{.}(2020b)]%
        {li2020real}
\bibfield{author}{\bibinfo{person}{Linfeng Li}, \bibinfo{person}{Peng Wang}, \bibinfo{person}{Jun Yan}, \bibinfo{person}{Yao Wang}, \bibinfo{person}{Simin Li}, \bibinfo{person}{Jinpeng Jiang}, \bibinfo{person}{Zhe Sun}, \bibinfo{person}{Buzhou Tang}, \bibinfo{person}{Tsung-Hui Chang}, \bibinfo{person}{Shenghui Wang}, {et~al\mbox{.}}} \bibinfo{year}{2020}\natexlab{b}.
\newblock \showarticletitle{Real-world data medical knowledge graph: construction and applications}.
\newblock \bibinfo{journal}{\emph{Artificial intelligence in medicine}}  \bibinfo{volume}{103} (\bibinfo{year}{2020}), \bibinfo{pages}{101817}.
\newblock


\bibitem[Li et~al\mbox{.}(2020a)]%
        {MemPool}
\bibfield{author}{\bibinfo{person}{Maosen Li}, \bibinfo{person}{Siheng Chen}, \bibinfo{person}{Ya Zhang}, {and} \bibinfo{person}{Ivor Tsang}.} \bibinfo{year}{2020}\natexlab{a}.
\newblock \showarticletitle{Graph cross networks with vertex infomax pooling}.
\newblock \bibinfo{journal}{\emph{Advances in Neural Information Processing Systems}}  \bibinfo{volume}{33} (\bibinfo{year}{2020}), \bibinfo{pages}{14093--14105}.
\newblock


\bibitem[Lin et~al\mbox{.}(2020)]%
        {lin2020kgnn}
\bibfield{author}{\bibinfo{person}{Xuan Lin}, \bibinfo{person}{Zhe Quan}, \bibinfo{person}{Zhi-Jie Wang}, \bibinfo{person}{Tengfei Ma}, {and} \bibinfo{person}{Xiangxiang Zeng}.} \bibinfo{year}{2020}\natexlab{}.
\newblock \showarticletitle{KGNN: Knowledge Graph Neural Network for Drug-Drug Interaction Prediction.}
\newblock \bibinfo{journal}{\emph{International Joint Conferences on Artificial Intelligence}} (\bibinfo{year}{2020}), \bibinfo{pages}{2739--2745}.
\newblock


\bibitem[Liu et~al\mbox{.}(2022)]%
        {liu2022graphpooling}
\bibfield{author}{\bibinfo{person}{Chuang Liu}, \bibinfo{person}{Yibing Zhan}, \bibinfo{person}{Jia Wu}, \bibinfo{person}{Chang Li}, \bibinfo{person}{Bo Du}, \bibinfo{person}{Wenbin Hu}, \bibinfo{person}{Tongliang Liu}, {and} \bibinfo{person}{Dacheng Tao}.} \bibinfo{year}{2022}\natexlab{}.
\newblock \showarticletitle{Graph pooling for graph neural networks: Progress, challenges, and opportunities}.
\newblock \bibinfo{journal}{\emph{arXiv preprint arXiv:2204.07321}} (\bibinfo{year}{2022}).
\newblock


\bibitem[Mahdisoltani et~al\mbox{.}(2013)]%
        {YAGO}
\bibfield{author}{\bibinfo{person}{Farzaneh Mahdisoltani}, \bibinfo{person}{Joanna Biega}, {and} \bibinfo{person}{Fabian~M Suchanek}.} \bibinfo{year}{2013}\natexlab{}.
\newblock \showarticletitle{Yago3: A knowledge base from multilingual wikipedias}. In \bibinfo{booktitle}{\emph{CIDR}}.
\newblock


\bibitem[Meilicke et~al\mbox{.}(2020)]%
        {AnyBURL}
\bibfield{author}{\bibinfo{person}{Christian Meilicke}, \bibinfo{person}{Melisachew~Wudage Chekol}, \bibinfo{person}{Manuel Fink}, {and} \bibinfo{person}{Heiner Stuckenschmidt}.} \bibinfo{year}{2020}\natexlab{}.
\newblock \showarticletitle{Reinforced anytime bottom up rule learning for knowledge graph completion}.
\newblock \bibinfo{journal}{\emph{arXiv preprint arXiv:2004.04412}} (\bibinfo{year}{2020}).
\newblock


\bibitem[Meilicke et~al\mbox{.}(2018)]%
        {RuleN}
\bibfield{author}{\bibinfo{person}{Christian Meilicke}, \bibinfo{person}{Manuel Fink}, \bibinfo{person}{Yanjie Wang}, \bibinfo{person}{Daniel Ruffinelli}, \bibinfo{person}{Rainer Gemulla}, {and} \bibinfo{person}{Heiner Stuckenschmidt}.} \bibinfo{year}{2018}\natexlab{}.
\newblock \showarticletitle{Fine-grained evaluation of rule-and embedding-based systems for knowledge graph completion}. In \bibinfo{booktitle}{\emph{Proceedings of International Semantic Web Conference}}. Springer, \bibinfo{pages}{3--20}.
\newblock


\bibitem[Norris(1998)]%
        {norris1998markov}
\bibfield{author}{\bibinfo{person}{James~R Norris}.} \bibinfo{year}{1998}\natexlab{}.
\newblock \bibinfo{booktitle}{\emph{Markov chains}}.
\newblock Number~2. \bibinfo{publisher}{Cambridge university press}.
\newblock


\bibitem[Qu and Tang(2019)]%
        {pLogicNet}
\bibfield{author}{\bibinfo{person}{Meng Qu} {and} \bibinfo{person}{Jian Tang}.} \bibinfo{year}{2019}\natexlab{}.
\newblock \showarticletitle{Probabilistic logic neural networks for reasoning}.
\newblock \bibinfo{journal}{\emph{Advances in Neural Information Processing Systems}}  \bibinfo{volume}{32} (\bibinfo{year}{2019}).
\newblock


\bibitem[Sadeghian et~al\mbox{.}(2019)]%
        {DRUM}
\bibfield{author}{\bibinfo{person}{Ali Sadeghian}, \bibinfo{person}{Mohammadreza Armandpour}, \bibinfo{person}{Patrick Ding}, {and} \bibinfo{person}{Daisy~Zhe Wang}.} \bibinfo{year}{2019}\natexlab{}.
\newblock \showarticletitle{Drum: End-to-end differentiable rule mining on knowledge graphs}.
\newblock \bibinfo{journal}{\emph{Advances in Neural Information Processing Systems}}  \bibinfo{volume}{32} (\bibinfo{year}{2019}).
\newblock


\bibitem[Schlichtkrull et~al\mbox{.}(2018)]%
        {RGCN}
\bibfield{author}{\bibinfo{person}{Michael Schlichtkrull}, \bibinfo{person}{Thomas~N Kipf}, \bibinfo{person}{Peter Bloem}, \bibinfo{person}{Rianne van~den Berg}, \bibinfo{person}{Ivan Titov}, {and} \bibinfo{person}{Max Welling}.} \bibinfo{year}{2018}\natexlab{}.
\newblock \showarticletitle{Modeling relational data with graph convolutional networks}. In \bibinfo{booktitle}{\emph{Proceedings of European Semantic Web Conference}}. \bibinfo{pages}{593--607}.
\newblock


\bibitem[Su et~al\mbox{.}(2023)]%
        {KRST}
\bibfield{author}{\bibinfo{person}{Zhixiang Su}, \bibinfo{person}{Di Wang}, \bibinfo{person}{Chunyan Miao}, {and} \bibinfo{person}{Lizhen Cui}.} \bibinfo{year}{2023}\natexlab{}.
\newblock \showarticletitle{Multi-Aspect Explainable Inductive Relation Prediction by Sentence Transformer}. In \bibinfo{booktitle}{\emph{Proceedings of the AAAI Conference on Artificial Intelligence}}. \bibinfo{pages}{6533--6540}.
\newblock


\bibitem[Teru et~al\mbox{.}(2020)]%
        {GRAIL}
\bibfield{author}{\bibinfo{person}{Komal Teru}, \bibinfo{person}{Etienne Denis}, {and} \bibinfo{person}{Will Hamilton}.} \bibinfo{year}{2020}\natexlab{}.
\newblock \showarticletitle{Inductive relation prediction by subgraph reasoning}. In \bibinfo{booktitle}{\emph{Proceedings of International Conference on Machine Learning}}. \bibinfo{pages}{9448--9457}.
\newblock


\bibitem[Toutanova and Chen(2015)]%
        {FB15k237}
\bibfield{author}{\bibinfo{person}{Kristina Toutanova} {and} \bibinfo{person}{Danqi Chen}.} \bibinfo{year}{2015}\natexlab{}.
\newblock \showarticletitle{Observed versus latent features for knowledge base and text inference}. In \bibinfo{booktitle}{\emph{Proceedings of the 3rd workshop on continuous vector space models and their compositionality}}. \bibinfo{pages}{57--66}.
\newblock


\bibitem[Vashishth et~al\mbox{.}(2019)]%
        {CompGCN}
\bibfield{author}{\bibinfo{person}{Shikhar Vashishth}, \bibinfo{person}{Soumya Sanyal}, \bibinfo{person}{Vikram Nitin}, {and} \bibinfo{person}{Partha Talukdar}.} \bibinfo{year}{2019}\natexlab{}.
\newblock \showarticletitle{Composition-based multi-relational graph convolutional networks}.
\newblock \bibinfo{journal}{\emph{arXiv preprint arXiv:1911.03082}} (\bibinfo{year}{2019}).
\newblock


\bibitem[Wang et~al\mbox{.}(2019)]%
        {wang2019kgat}
\bibfield{author}{\bibinfo{person}{Xiang Wang}, \bibinfo{person}{Xiangnan He}, \bibinfo{person}{Yixin Cao}, \bibinfo{person}{Meng Liu}, {and} \bibinfo{person}{Tat-Seng Chua}.} \bibinfo{year}{2019}\natexlab{}.
\newblock \showarticletitle{Kgat: Knowledge graph attention network for recommendation}. In \bibinfo{booktitle}{\emph{Proceedings of the 25th ACM SIGKDD international conference on knowledge discovery \& data mining}}. \bibinfo{pages}{950--958}.
\newblock


\bibitem[Wu et~al\mbox{.}(2020)]%
        {wu2020comprehensive}
\bibfield{author}{\bibinfo{person}{Zonghan Wu}, \bibinfo{person}{Shirui Pan}, \bibinfo{person}{Fengwen Chen}, \bibinfo{person}{Guodong Long}, \bibinfo{person}{Chengqi Zhang}, {and} \bibinfo{person}{S~Yu Philip}.} \bibinfo{year}{2020}\natexlab{}.
\newblock \showarticletitle{A comprehensive survey on graph neural networks}.
\newblock \bibinfo{journal}{\emph{IEEE transactions on neural networks and learning systems}} \bibinfo{volume}{32}, \bibinfo{number}{1} (\bibinfo{year}{2020}), \bibinfo{pages}{4--24}.
\newblock


\bibitem[Xiong et~al\mbox{.}(2017)]%
        {NELL995}
\bibfield{author}{\bibinfo{person}{Wenhan Xiong}, \bibinfo{person}{Thien Hoang}, {and} \bibinfo{person}{William~Yang Wang}.} \bibinfo{year}{2017}\natexlab{}.
\newblock \showarticletitle{Deeppath: A reinforcement learning method for knowledge graph reasoning}.
\newblock \bibinfo{journal}{\emph{arXiv preprint arXiv:1707.06690}} (\bibinfo{year}{2017}).
\newblock


\bibitem[Xu et~al\mbox{.}(2018)]%
        {xu2018powerful}
\bibfield{author}{\bibinfo{person}{Keyulu Xu}, \bibinfo{person}{Weihua Hu}, \bibinfo{person}{Jure Leskovec}, {and} \bibinfo{person}{Stefanie Jegelka}.} \bibinfo{year}{2018}\natexlab{}.
\newblock \showarticletitle{How powerful are graph neural networks?}
\newblock \bibinfo{journal}{\emph{arXiv preprint arXiv:1810.00826}} (\bibinfo{year}{2018}).
\newblock


\bibitem[Xu et~al\mbox{.}(2019)]%
        {DPMPN}
\bibfield{author}{\bibinfo{person}{Xiaoran Xu}, \bibinfo{person}{Wei Feng}, \bibinfo{person}{Yunsheng Jiang}, \bibinfo{person}{Xiaohui Xie}, \bibinfo{person}{Zhiqing Sun}, {and} \bibinfo{person}{Zhi-Hong Deng}.} \bibinfo{year}{2019}\natexlab{}.
\newblock \showarticletitle{Dynamically pruned message passing networks for large-scale knowledge graph reasoning}.
\newblock \bibinfo{journal}{\emph{arXiv preprint arXiv:1909.11334}} (\bibinfo{year}{2019}).
\newblock


\bibitem[Yang et~al\mbox{.}(2017)]%
        {NeuralLP}
\bibfield{author}{\bibinfo{person}{Fan Yang}, \bibinfo{person}{Zhilin Yang}, {and} \bibinfo{person}{William~W Cohen}.} \bibinfo{year}{2017}\natexlab{}.
\newblock \showarticletitle{Differentiable learning of logical rules for knowledge base reasoning}.
\newblock \bibinfo{journal}{\emph{Advances in Neural Information Processing Systems}}  \bibinfo{volume}{30} (\bibinfo{year}{2017}).
\newblock


\bibitem[Ying et~al\mbox{.}(2018)]%
        {DiffPool}
\bibfield{author}{\bibinfo{person}{Zhitao Ying}, \bibinfo{person}{Jiaxuan You}, \bibinfo{person}{Christopher Morris}, \bibinfo{person}{Xiang Ren}, \bibinfo{person}{Will Hamilton}, {and} \bibinfo{person}{Jure Leskovec}.} \bibinfo{year}{2018}\natexlab{}.
\newblock \showarticletitle{Hierarchical graph representation learning with differentiable pooling}.
\newblock \bibinfo{journal}{\emph{Advances in Neural Information Processing Systems}}  \bibinfo{volume}{31} (\bibinfo{year}{2018}).
\newblock


\bibitem[Zhang and Yao(2022)]%
        {RED-GNN}
\bibfield{author}{\bibinfo{person}{Yongqi Zhang} {and} \bibinfo{person}{Quanming Yao}.} \bibinfo{year}{2022}\natexlab{}.
\newblock \showarticletitle{Knowledge graph reasoning with relational digraph}. In \bibinfo{booktitle}{\emph{Proceedings of the ACM web conference}}. \bibinfo{pages}{912--924}.
\newblock


\bibitem[Zhang et~al\mbox{.}(2022)]%
        {zhang2022rethinking}
\bibfield{author}{\bibinfo{person}{Zhanqiu Zhang}, \bibinfo{person}{Jie Wang}, \bibinfo{person}{Jieping Ye}, {and} \bibinfo{person}{Feng Wu}.} \bibinfo{year}{2022}\natexlab{}.
\newblock \showarticletitle{Rethinking graph convolutional networks in knowledge graph completion}. In \bibinfo{booktitle}{\emph{Proceedings of the ACM Web Conference}}. \bibinfo{pages}{798--807}.
\newblock


\bibitem[Zhu et~al\mbox{.}(2021)]%
        {NBFNet}
\bibfield{author}{\bibinfo{person}{Zhaocheng Zhu}, \bibinfo{person}{Zuobai Zhang}, \bibinfo{person}{Louis-Pascal Xhonneux}, {and} \bibinfo{person}{Jian Tang}.} \bibinfo{year}{2021}\natexlab{}.
\newblock \showarticletitle{Neural bellman-ford networks: A general graph neural network framework for link prediction}.
\newblock \bibinfo{journal}{\emph{Advances in Neural Information Processing Systems}}  \bibinfo{volume}{34} (\bibinfo{year}{2021}), \bibinfo{pages}{29476--29490}.
\newblock


\end{thebibliography}
\begin{appendix}
\section{Proof of Theorem}\label{sec_appendix_proof}
    \begin{myTheo}\label{theo_propto_appendix}
        Given Assumption~\ref{assum_independent}, for a query relation $r$ and $\forall r_1, r_2 \in \textit{NR}_r$, the relationship between the combined and individual $\textit{Rel}^{\textit{pre}}$ and $\textit{Rel}^{\textit{rec}}$ can be expressed as follows:
        \begin{equation}\label{eq_propto_theo_1_appendix}
        \begin{aligned}
            \textit{Rel}^{\textit{pre}}(\{r_1, r_2\},r) \propto \textit{Rel}^{\textit{pre}}(\{r_1\},r)*\textit{Rel}^{\textit{pre}}(\{r_2\},r), \\
            \textit{Rel}^{\textit{rec}}(\{r_1, r_2\},r) \propto \textit{Rel}^{\textit{rec}}(\{r_1\},r)*\textit{Rel}^{\textit{rec}}(\{r_2\},r).
        \end{aligned}
        \end{equation}
        Furthermore, for any subset $\textit{NR}' \subseteq \textit{NR}_r$, the following can be deduced:
        \begin{equation}\label{eq_propto_theo_2_appendix}
        \begin{aligned}
            \textit{Rel}^{\textit{pre}}(\textit{NR}',r) \propto \prod_{r_i \in \textit{NR}'} \textit{Rel}^{\textit{pre}}(\{r_i\},r), \\
            \textit{Rel}^{\textit{rec}}(\textit{NR}',r) \propto \prod_{r_i \in \textit{NR}'} \textit{Rel}^{\textit{rec}}(\{r_i\},r).
        \end{aligned}
        \end{equation}
    \end{myTheo}
    \begin{proof}
    \begin{equation}
        \begin{aligned}
        \textit{Rel}^{\textit{pre}}(\{r_1, r_2\},r) 
        =& P(r|r_1,r_2) \\
        =& \frac{P(r_1,r_2,r)}{P(r_1,r_2)} \\
        =& \frac{P(r_1,r_2,r)}{P(r)} \frac{P(r)}{P(r_1)P(r_2)} \\
        =& P(r_1,r_2|r) \frac{P(r)}{P(r_1)P(r_2)} \\
        =& P(r_1|r)P(r_2|r)\frac{P(r)}{P(r_1)P(r_2)} \\
        =& \frac{P(r_1,r)}{P(r)} \frac{P(r_2,r)}{P(r)} \frac{P(r)}{P(r_1)P(r_2)} \\
        =& \frac{P(r_1,r)}{P(r_1)} \frac{P(r_2,r)}{P(r_2)} \frac{1}{P(r)} \\
        =& \frac{P(r|r_1)P(r|r_2)}{P(r)} \\
        =& \frac{\textit{Rel}^{\textit{pre}}(\{r_1\},r)*\textit{Rel}^{\textit{pre}}(\{r_2\},r)}{P(r)}.
        \end{aligned}
    \end{equation}
    
    \begin{equation}
        \begin{aligned}
        \textit{Rel}^{\textit{rec}}(\{r_1, r_2\},r) 
        =& P(r_1,r_2|r) \\
        =& P(r_1|r)P(r_2|r) \\
        =& \textit{Rel}^{\textit{rec}}(\{r_1\},r)*\textit{Rel}^{\textit{rec}}(\{r_2\},r).
        \end{aligned}
    \end{equation}
    Utilizing the principle of mathematical induction on~(\ref{eq_propto_theo_1_appendix}) enables us to logically derive the validity of~(\ref{eq_propto_theo_2_appendix}).
    \end{proof}
\vfill
\section{Datasets}\label{sec_dataset_appendix}
    The numbers of relations ($|R_G|$), entities ($|E_G|$) and triples (\#Triples) of the datasets used in this work are shown in Table~\ref{table_dataset_appendix}. In both transductive and inductive settings, the datasets WN18RR-V3, FB15k-237-V4, and NELL-995-V3 are typically recognized as large, comprising over 20,000 training and testing triples.
    \begin{table}[H]
        \centering
        \caption{Statistics of training (testing) datasets}\label{table_dataset_appendix}
        \begin{tabular}{lccc}
        \toprule
        $G$      & $|R_G|$   & $|E_G|$      & \#Triples    \\
        \midrule
        \multicolumn{4}{c}{\textbf{WN18RR}}                          \\ \midrule
        Trans-V1 & 9 (7)     & 2,746 (962)    & 5,410 (638)     \\
        Trans-V2 & 10 (9)    & 6,954 (2,788)  & 15,262 (1,868)  \\
        Trans-V3 & 11 (10)   & 12,078 (4,605) & 25,901 (3,152)  \\
        Trans-V4 & 9 (8)     & 3,861 (1,433)  & 7,940 (968)     \\
        Ind-V1   & 9 (9)     & 2,746 (922)    & 6,678 (1,911)   \\
        Ind-V2   & 10 (10)   & 6,954 (2,923)  & 18,968 (4,863)  \\
        Ind-V3   & 11 (11)   & 12,078 (5,084) & 32,150 (7,470)  \\
        Ind-V4   & 9 (9)     & 3,861 (7,208)  & 9,842 (15,157)  \\ \midrule
        \multicolumn{4}{c}{\textbf{FB15k-237}}                  \\ \midrule
        Trans-V1 & 180 (102) & 1,594 (550)    & 4,245 (492)     \\
        Trans-V2 & 200 (140) & 2,608 (1,142)  & 9,739 (1,180)   \\
        Trans-V3 & 215 (179) & 3,668 (1,871)  & 17,986 (2,214)  \\
        Trans-V4 & 219 (192) & 4,707 (2,627)  & 27,203 (3,361)  \\
        Ind-V1   & 183 (146) & 2,000 (1,500)  & 5,226 (2,404)   \\
        Ind-V2   & 203 (176) & 3,000 (2,000)  & 12,085 (5,092)  \\
        Ind-V3   & 218 (187) & 4,000 (3,000)  & 22,394 (9,137)  \\
        Ind-V4   & 222 (204) & 5,000 (3,500)  & 33,916 (14,554) \\ \midrule
        \multicolumn{4}{c}{\textbf{NELL-995}}                   \\ \midrule
        Trans-V1 & 14 (14)   & 3,103 (553)    & 4,687 (439)     \\
        Trans-V2 & 88 (60)   & 2,564 (841)    & 8,219 (968)     \\
        Trans-V3 & 142 (94)  & 4,647 (1,473)  & 16,393 (1,873)  \\
        Trans-V4 & 76 (46)   & 2,092 (699)    & 7,546 (867)     \\
        Ind-V1   & 14 (14)   & 10,915 (225)   & 5,540 (1,034)   \\
        Ind-V2   & 88 (79)   & 2,564 (4,937)  & 10,109 (5,521)  \\
        Ind-V3   & 142 (122) & 4,047 (4,921)  & 20,117 (9,668)  \\
        Ind-V4   & 77 (61)   & 2,092 (3,294)  & 9,089 (8,520)   \\
        \bottomrule
        \end{tabular}
    \end{table}

\end{appendix}
\balance
\end{document}